\documentclass[accepted]{uai2021}

\usepackage{times}
\usepackage{amsthm}
\usepackage{mathtools} 
\usepackage{booktabs} 

\usepackage[american]{babel}
\usepackage[square,sort,comma,numbers]{natbib} 

\usepackage{amsfonts}
\usepackage{booktabs}
\usepackage{siunitx}
\usepackage{multirow}
\usepackage{graphicx}
\usepackage{listings}
\usepackage{commath}
\usepackage{epstopdf}
\usepackage{xcolor}
\usepackage{url}
\usepackage{mathtools}
\usepackage{caption}
\usepackage{alltt}
\usepackage{mathrsfs}
\usepackage{float}
\usepackage{caption}
\usepackage[normalem]{ulem}
\usepackage{graphicx,fancyvrb}
\usepackage[colorinlistoftodos]{todonotes}
\usepackage[linesnumbered,ruled,vlined]{algorithm2e}
\usepackage{algpseudocode}
\usepackage{amsmath,amssymb}
\usepackage{enumitem}
\usepackage{booktabs}
\usepackage{caption}
\usepackage{float}
\usepackage{capt-of}
\usepackage{multirow}

\usepackage{array}
\usepackage{arydshln}
\usepackage{todonotes}

\definecolor{mygreen}{rgb}{0,0.6,0}
\definecolor{mygray}{rgb}{0.5,0.5,0.5}
\definecolor{mymauve}{rgb}{0.58,0,0.82}
\definecolor{cadmiumgreen}{rgb}{0.0, 0.42, 0.24}

\usepackage{listings}
\lstnewenvironment{VerbatimText}[1][]{
    
    \lstset{fancyvrb=true,basicstyle=\footnotesize,captionpos=b,xleftmargin=2em,#1}
}{}

\newcommand{\ignore}[1]{}

\newcommand{\fd}{\rightarrow}

\newcommand{\real}{{\mathbb{R}}}

\newcommand*{\rom}[1]{\expandafter\@slowromancap\romannumeral #1@}
\newcommand{\RNum}[1]{\uppercase\expandafter{\romannumeral #1\relax}}

\newtheorem{defn}{Definition}[section]
\newtheorem{example}[defn]{Example}
\newtheorem{corr}[defn]{Corollary}
\newtheorem{lemma}[defn]{Lemma}
\newtheorem{thm}[defn]{Theorem}
\newtheorem{theorem}[defn]{Theorem}

\newcommand{\sel}[1]{{\sigma}}

\newcommand{\cut}[1]{}
\newcommand{\eat}[1]{}

\newcommand{\defeq}{\stackrel{\text{def}}{=}}

\def\set#1{\mathord{\{#1\}}}

\def\iset{\mathrm{m}}
\def\isetc{\mathrm{m}^c}

\def\measure{\mu}
\def\imeasure{\mu^*}
\def\eqdef{\mathrel{\stackrel{\textsf{\tiny def}}{=}}}

\def\D{\mathcal{D}}
\def\e#1{\emph{#1}}
\newenvironment{citedtheorem}[1]
{\begin{thm}{\it\e{(#1)}}\,\,}
	{\end{thm}}

\newenvironment{citeddefn}[1]
{\begin{defn}{\it\e{(#1)}}\,\,}
	{\end{defn}}
\def\implies{\Rightarrow}

\newcommand{\pow}[1]{2^{{#1}}} 

\def\vplus{{+}}

\newenvironment{repeatresult}[2]
{\vskip0.5em\par\textsc{#1} #2.\em}
{\vskip1em}

\def\appendix{\par
	\section*{APPENDIX}
	\setcounter{section}{0}
	\setcounter{subsection}{0}
	\def\thesection{\Alph{section}} }

\def\entropicPlhdrl{\Gamma}

\def\iset{\mathrm{m}}
\def\isetc{\mathrm{m}^c}

\def\measure{\mu}
\def\imeasure{\mu^*}
\def\eqdef{\mathrel{\stackrel{\textsf{\tiny def}}{=}}}

\def\e#1{\emph{#1}}

\def\impliedCI{\tau}

\def\bd{\boldsymbol{d}}

\def\vars{\mathbf{var}}

\title{Approximate Implication with d-Separation}
\author{Batya Kenig\\
Technion, Israel Institute of Technology\\ Haifa, Israel\\
batyak@technion.ac.il
}

\begin{document}

\maketitle

\begin{abstract}
The graphical structure of Probabilistic Graphical Models (PGMs) encodes the conditional independence (CI) relations that hold in the modeled distribution. Graph algorithms, such as \e{d-separation}, use this structure to infer additional conditional independencies, and to query whether a specific CI holds in the distribution.\eat{These CIs are what make PGMs an efficient tool for representation, inference and learning in a wide range of applications.} 
The premise of all current systems-of-inference for deriving CIs in PGMs, is that the set of CIs used for the construction of the PGM hold \e{exactly}.
In practice, algorithms for extracting the structure of PGMs from data, discover \e{approximate CIs} that do not hold exactly in the distribution. In this paper, we ask how the error in this set propagates to the inferred CIs\eat{, or implied , discovered by current systems of inference} read off the graphical structure. More precisely, what guarantee can we provide on the inferred CI when the set of CIs that entailed it hold only approximately? It has recently been shown that in the general case, no such guarantee can be provided. We prove that such a guarantee exists for the set of CIs inferred in directed graphical models, making the \e{$d$-separation} algorithm a sound and complete system for inferring \e{approximate CIs}. We also prove an approximation guarantee for independence relations derived from \e{marginal} CIs.

\end{abstract}

\section{INTRODUCTION}
Conditional independencies (CI) are assertions of the form $X\bot Y|Z$, stating that the random variables (RVs) $X$ and $Y$ are independent when conditioned on $Z$. The concept of conditional independence is at the core of Probabilistic graphical Models (PGMs) that include Bayesian and Markov networks. The CI relations between the random variables \eat{are those that} enable the modular and low-dimensional representations of high-dimensional, multivariate distributions, and\eat{enable efficient} tame the complexity of inference and learning, which would otherwise be very inefficient~\cite{DBLP:books/daglib/0023091,DBLP:books/daglib/0066829}.

The \e{implication problem} is the task of determining whether a set of CIs termed \e{antecedents} logically entail another CI, called the \e{consequent}, and it has received considerable attention from both the AI and Database communities~\cite{DBLP:conf/ecai/PearlP86,DBLP:conf/uai/GeigerVP89,DBLP:journals/iandc/GeigerPP91,SAYRAFI2008221,DBLP:conf/icdt/KenigS20,DBLP:conf/sigmod/KenigMPSS20}. Known algorithms for deriving CIs from the topological structure of the graphical model are, in fact, an instance of implication. 
Notably, the DAG structure of Bayesian Networks is generated based on a set of CIs termed the \e{recursive basis}~\cite{DBLP:journals/networks/GeigerVP90}, and the $d$-separation algorithm is used to derive additional CIs, implied by this set.
The $d$-separation algorithm is a sound and complete method for deriving CIs in probability distributions represented by DAGs~\cite{DBLP:conf/uai/GeigerVP89, DBLP:journals/networks/GeigerVP90}, and hence completely characterizes the CIs that hold in the distribution. The foundation of deriving CIs in both directed and undirected models is the \e{semigraphoid axioms}~\cite{Dawid1979,GEIGER1991128,GeigerPearl1993}.
\eat{
More precisely, the soundness of $d$-separation was established by Verma, who also showed that the set of CIs derived by $d$-separation is precisely the closure of the recursive basis under the \e{semgraphoid axioms}~\cite{GEIGER1991128,GeigerPearl1993}. Then, Geiger and Pearl proved that the semigraphoid axioms are a complete system-of-inference for deriving CIs from the recursive basis, thus showing that $d$-separation is complete. }

Current systems for inferring CIs, and the semigraphoid axioms in particular, assume that both antecedents and consequent hold \e{exactly}, hence we refer to these as an exact implication (EI). 
However, almost all known approaches for learning the structure of a PGM rely on CIs extracted from data, which hold to a large degree, but cannot be expected to hold exactly. Of these, structure-learning approaches based on information theory have been shown to be particularly successful, and thus widely used to infer networks in many fields~\cite{CHENG200243,JMLR:v7:decampos06a,Chen2008TKDE,Zhao5130,DBLP:conf/sigmod/KenigMPSS20}. 

In this paper, we drop the assumption that the CIs hold exactly, and
consider the \e{relaxation
problem}: if an exact implication holds, does an \e{approximate implication} hold too? 
That is, if the antecedents approximately hold in the distribution, does the consequent approximately hold as well ?
What guarantees can we give for the approximation? In other words, the relaxation problem
asks whether we can convert an exact implication to an approximate one. When relaxation holds, then any system-of-inference for deriving exact implications, (e.g. the semigraphoid axioms, $d$-separation), can be used to infer approximate implications as well. 

To study the relaxation problem we need to measure the degree of satisfaction of
a CI. In line with previous work, we use Information Theory. This is the natural semantics for
modeling CIs because $X \bot Y | Z$ if and only if $I(X; Y |Z) = 0$, where $I$ is the
conditional mutual information. Hence, an exact implication (EI) $\sigma_1,\cdots,\sigma_k \implies \tau$ is an assertion of the form $(h(\sigma_1){=}0 \wedge  \cdots \wedge h(\sigma_k){=}0) \implies
h(\tau){=}0$, where 
$\tau,\sigma_1,\sigma_2,\dots$ are triples $(X;Y|Z)$, and
$h$ is the conditional mutual information measure $I(\cdot;\cdot|\cdot)$. An approximate implication (AI) is a linear inequality $h(\tau) \leq \lambda h(\Sigma)$, where $h(\Sigma)\eqdef\sum_{i=1}^kh(\sigma_i)$, and $\lambda \geq 0$ is the approximation factor. We say that a class of CIs \e{$\lambda$-relaxes} if every exact implication (EI) from the class can be transformed to an approximate implication (AI) with an approximation factor $\lambda$. We observe that an approximate implication always implies an exact implication\eat{, even if we do not have an upper bound on $\lambda$, but every concrete AI has a finite lambda This is} because the mutual information $I(\cdot;\cdot|\cdot)\geq 0$ is a nonnegative measure. Therefore, if $0\leq h(\tau)\leq \lambda h(\Sigma)$ for some $\lambda \geq 0$, then $h(\Sigma)=0 \implies h(\tau)=0$.
 \eat{\footnote{This is because $I(\cdot;\cdot|\cdot)\geq 0$ is a nonnegative measure. Therefore, if $0\leq h(\tau)\leq \lambda h(\Sigma)$, then $h(\Sigma)=0 \implies h(\tau)=0$. }. }

\textbf{Results.} 
A conditional independence assertion $(A;B|C)$ is called \e{saturated} if it mentions all of the random variables in the distribution, and it is called \e{marginal} if $C=\emptyset$. 

We show that every conditional independence relation $(X;Y|Z)$ read off a DAG by the d-separation algorithm~\cite{DBLP:conf/uai/GeigerVP89}, admits a $1$-approximation. In other words, if $\Sigma$ is the \e{recursive basis} of CIs used to build the Bayesian network~\cite{DBLP:conf/uai/GeigerVP89}, then it is guaranteed that $I(X;Y|Z)\leq \sum_{i\in \Sigma}h(\sigma_i)$. Furthermore, we present a family of implications for which our $1$-approximation is tight (i.e., $I(X;Y|Z)=\sum_{i\in \Sigma}h(\sigma_i)$). We also prove that every CI $(X;Y|Z)$ implied by a set of marginal CIs admits an  $|X|\cdot|Y|$-approximation (i.e., where $|X|$ denotes the number of RVs in the set $X$). The exact variant of implication from these classes of CIs were extensively studied ~\cite{DBLP:conf/uai/GeigerVP89, GeigerPearl1993,DBLP:conf/uai/GeigerP88,DBLP:journals/iandc/GeigerPP91,DBLP:journals/networks/GeigerVP90} (see below the related work). Here, we study their approximation.

Of independent interest is the technique used for proving the approximation guarantees. The \e{I-measure}~\cite{DBLP:journals/tit/Yeung91} is a theory which establishes a one-to-one correspondence between information theoretic measures such as entropy and mutual information (defined in Section~\ref{sec:notations}) and set theory. Ours is the first to apply this technique to the study of CI implication. 

\textbf{Related Work.}
The AI community has extensively studied
the exact implication problem for Conditional Independencies (CI).
In a series of papers, Geiger et al. showed that the \e{semigraphoid axioms}~\cite{DBLP:conf/ecai/PearlP86} are sound and complete for deriving CI statements that are implied by saturated CIs~\cite{GeigerPearl1993}, marginal CIs~\cite{GeigerPearl1993}, and \e{recursive CIs} that are used in Bayesian networks~\cite{DBLP:journals/networks/GeigerVP90,DBLP:conf/uai/GeigerP88}.
The completeness of $d$-separation follows from the fact that the set of CIs derived by $d$-separation is precisely the closure of the recursive basis under the \e{semgraphoid axioms}~\cite{VERMA199069}. Studen{\'{y}}
proved that in the general case, when no assumptions are made on the antecendents, no finite axiomatization exists~\cite{StudenyCINoCharacterization1990}. That is, there does not exist a finite set of axioms (deductive rules) from which all general conditional independence implications can be deduced. 

The database community has also studied the EI problem for integrity constraints~\cite{DBLP:journals/tods/ArmstrongD80,DBLP:conf/sigmod/BeeriFH77,10.1007/978-3-642-39992-3_17,Maier:1983:TRD:1097039},
and showed that the implication problem is decidable and axiomatizable when the antecedents are Functional
Dependencies or \e{Multivalued Dependencies} (which correspond to saturated CIs, see~\cite{DBLP:journals/tse/Lee87,DBLP:conf/icdt/KenigS20}), and undecidable
for \e{Embedded Multivalued Dependencies}~\cite{10.1006/inco.1995.1148}. 

The relaxation problem was first studied by Kenig and Suciu in the context of database dependencies~\cite{DBLP:conf/icdt/KenigS20}, where they showed that CIs derived from a set of saturated antecedents, admit an approximate implication. Importantly, they also showed that not all exact implications relax, and presented a family of 4-variable distributions along with an exact implication that does not admit an approximation (see Theorem 16 in~\cite{DBLP:conf/icdt/KenigS20}).  Consequently, it is not straightforward that exact implication necessarily imply its approximation counterpart, and arriving at meaningful approximation guarantees requires making certain assumptions on the antecedents, consequent, or both.

\textbf{Organization.}
We start in Section~\ref{sec:notations} with preliminaries. We formally define the relaxation problem in Section~\ref{sec:problem:def}, and formally state our results in Section~\ref{sec:results}. In Section~\ref{sec:lemmas} we establish, through a series of lemmas, properties of exact implication that will be used for proving our results. In Section~\ref{sec:recursiveProof} we prove that every implication from a set of recursive CIs admits a 1-relaxation, and in Section~\ref{sec:marginalProof} we prove that every implication $\Sigma\implies (X;Y|Z)$ from a set of marginal CIs admits an $|X||Y|$-relaxation. We conclude in Section~\ref{sec:conclusion}.
\section{Preliminaries}
\label{sec:notations}

We denote by $[n] = \set{1,2,\ldots,n}$.  If
$\Omega=\set{X_1,\ldots, X_n}$ denotes a set of variables and
$U, V \subseteq \Omega$, then we abbreviate the union $U \cup V$ with
$UV$. 

\subsection{Conditional Independence}
Recall that two discrete random variables $X, Y$ are called {\em
	independent} if $p(X=x, Y=y) = p(X=x)\cdot p(Y=y)$ for all outcomes
$x,y$. Fix $\Omega=\set{X_1,\dots,X_n}$, a set of $n$ jointly
distributed discrete random variables with finite domains
$\D_1,\dots,\D_n$, respectively; let $p$ be the probability mass.  For
$\alpha \subseteq [n]$, denote by $X_\alpha$ the joint random variable
$(X_i: i \in \alpha)$ with domain
$\D_\alpha \defeq \prod_{i \in \alpha} D_i$.  We write
$p \models X_\beta \perp X_\gamma | X_\alpha$ when $X_\beta, X_\gamma$
are conditionally independent given $X_\alpha$; in the special case
that $X_\alpha$ functionally determines $X_\beta$, we write $p \models X_\alpha \fd X_\beta$.
\eat{
$\beta=\gamma$, then $p \models X_\beta\perp X_\beta | X_\alpha$ iff
$X_\alpha$ functionally determines\footnote{This means:
	$\forall u \in D_\alpha$, if $p(X_\alpha = u) \neq 0$ then
	$\exists v \in D_\beta$ s.t. $p(X_\beta=v|X_\alpha=u)=1$, and $v$ is
	unique.}  $X_\beta$, and we write $p \models X_\alpha \fd X_\beta$.
}

An assertion $X {\perp} Y | Z$ is called a {\em Conditional
	Independence} statement, or a CI; this includes $Z \rightarrow Y$ as
a special case.  When $XYZ=\Omega$ we call it \e{saturated}, and when $Z=\emptyset$ we call it \e{marginal}.  A set of
CIs $\Sigma$ {\em implies} a CI $\tau$, in notation
$\Sigma \Rightarrow \tau$, if every probability distribution that
satisfies $\Sigma$ also satisfies $\tau$.
\eat{  
This implication problem
has also been extensively studied: 
In a series of papers, Geiger et al. showed that the \e{semigraphoid axioms}~\cite{DBLP:conf/ecai/PearlP86} are sound and complete for deriving CI statements that are implied by saturated CIs~\cite{GeigerPearl1993}, marginal CIs~\cite{GeigerPearl1993}, and \e{recursive CIs} that are used in Bayesian networks~\cite{DBLP:journals/networks/GeigerVP90,DBLP:conf/uai/GeigerP88}. Studen{\'{y}}
proved that in the general case, when no assumptions are made on the antecendents, no finite axiomatization exists~\cite{StudenyCINoCharacterization1990}.
}

\subsection{Background on Information Theory}

\label{subsec:information:theory}

We adopt required notation from the literature on information
theory~\cite{Yeung:2008:ITN:1457455}.  For $n > 0$, we
identify the functions
$\pow{[n]}\rightarrow \real$ with the vectors in $\real^{2^n}$.

\noindent {\bf Polymatroids.} A function\eat{\footnote{Most authors
	consider rather the space $\real^{2^n-1}$, by dropping
	$h(\emptyset)$ because it is always $0$.}} $h \in \real^{2^n}$ is
called a \emph{polymatroid} if $h(\emptyset)=0$ and satisfies the
following inequalities, called {\em Shannon inequalities}:
\begin{enumerate}
	\item Monotonicity: $h(A)\leq h(B)$ for $A \subseteq B$
	\item Submodularity: $h(A\cup B)+h(A\cap B)\leq h(A) + h(B)$ for all $A,B \subseteq [n]$
\end{enumerate}
The set of polymatroids is denoted $\Gamma_n \subseteq \real^{2^n}$.  For any polymatroid $h$ and subsets $A,B,C,D \subseteq [n]$, we define\footnote{Recall that $AB$
	denotes $A \cup B$.}
\begin{align}
	h(B|A) \eqdef &~h(AB) - h(A) \label{eq:h:cond} \\
	I_h(B;C|A) \eqdef &~h(AB) + h(AC) - h(ABC) - h(A) \label{eq:h:mutual:information}
\end{align}

Then, $\forall h\in \Gamma_n$, $I_h(B;C|A) \geq 0$ by submodularity, and
$h(B|A) \geq 0$ by monotonicity. We say that $A$ \e{functionally determines} $B$, in notation $A \fd B$ if $h(B|A)=0$. The \e{chain rule} is the identity:
\begin{equation} \label{eq:ChainRuleMI}
	I_h(B;CD|A)=I_h(B;C|A)+I_h(B;D|AC)
\end{equation}
We call the triple $(B;C|A)$ \e{elemental}
if $|B|=|C|=1$; $h(B|A)$ is a special case of $I_h$, because
$h(B|A) = I_h(B;B|A)$. By the chain rule, it follows that every CI $(B;C|A)$ can be written as a sum of at most $|B||C|\leq \frac{n^2}{4}$ elemental CIs.

\noindent {\bf Entropy.} If $X$ is a random variable with
a finite domain $\D$ and probability mass $p$, then $H(X)$ denotes its
entropy
\begin{equation}\label{eq:entropy}
	H(X)\eqdef\sum_{x\in \D}p(x)\log\frac{1}{p(x)}
\end{equation}
For a set of jointly distributed random variables
$\Omega=\set{X_1,\dots,X_n}$ we define the function
$h : \pow{[n]} \rightarrow \real$ as $h(\alpha) \eqdef H(X_\alpha)$;
$h$ is called an \e{entropic function}, or, with some abuse, an
\e{entropy}. It is easily verified that the entropy $H$ satisfies the Shannon inequalities, and is thus a polymatroid. 
The quantities $h(B|A)$ and $I_h(B;C|A)$ are called the \e{conditional
	entropy} and \e{conditional mutual information} respectively.  The
conditional independence $p \models B \perp C \mid A$ holds iff
$I_h(B;C|A)=0$, and similarly $p \models A \fd B$ iff $h(B|A)=0$,
thus, entropy provides us with an alternative characterization of
CIs.

\subsubsection{The I-measure}\label{sec:imeasure}
The I-measure~\cite{DBLP:journals/tit/Yeung91,Yeung:2008:ITN:1457455} is a theory which establishes a one-to-one correspondence between Shannon's information measures and set theory. 
Let $h\in \entropicPlhdrl_n$ denote a polymatroid defined over the variables $\set{X_1,\dots,X_n}$. Every variable $X_i$ is associated with a set $\iset(X_i)$, and it's complement $\isetc(X_i)$.
The universal set is $\Lambda \eqdef \bigcup_{i=1}^n\iset(X_i)$.
Let $\alpha \subseteq [n]$. We denote by $X_\alpha\eqdef\set{X_j \mid j \in \alpha}$, and $\iset(X_\alpha)\eqdef \bigcup_{i\in \alpha}\iset(X_i)$.

\begin{citeddefn}{\cite{DBLP:journals/tit/Yeung91,Yeung:2008:ITN:1457455}} \label{thm:YeungUniqueness}\label{def:field}
	The field $\mathcal{F}_n$ generated by sets $\iset(X_1),\dots,\iset(X_n)$ is the collection of sets which can be obtained by any sequence of usual set operations (union, intersection, complement, and difference) on $\iset(X_1),\dots,\iset(X_n)$.
\end{citeddefn}

The \e{atoms} of $\mathcal{F}_n$ are sets of the form $\bigcap_{i=1}^nY_i$, where $Y_i$ is either $\iset(X_i)$ or $\isetc(X_i)$. We denote by $\mathcal{A}$ the atoms of $\mathcal{F}_n$.
We consider only atoms in which at least one set appears in positive form (i.e., the atom $\bigcap_{i=1}^n\isetc(X_i)\eqdef \emptyset$ is empty).
There are $2^n-1$ non-empty atoms and $2^{2^n-1}$ sets in $\mathcal{F}_n$ expressed as the union of its atoms. 
A function $\measure:\mathcal{F}_n \rightarrow \real$ is \e{set additive} if for every pair of disjoint sets $A$ and $B$ it holds that $\measure(A\cup B)=\measure(A)\vplus\measure(B)$.
A real function $\measure$ defined on $\mathcal{F}_n$ is called a \e{signed measure} if it is set additive, and $\measure(\emptyset)=0$.

The $I$-measure $\imeasure$ on $\mathcal{F}_n$ is defined by $\imeasure(m(X_\alpha))=H(X_\alpha)$ for all nonempty subsets $\alpha \subseteq \set{1,\dots,n}$, where $H$ is the entropy~\eqref{eq:entropy}. 
Table~\ref{tab:ImeasureSummary} summarizes the extension of this definition to the rest of the Shannon measures.
\begin{table}[]
	\centering
	\small
	\begin{tabular}{|c|c|}	
		\hline
		Information & \multirow{2}{*}{$\imeasure$} \\
		Measures &  \\ \hline
		$H(X)$	& $\imeasure(\iset(X))$ \\ \hline
		$H(XY)$	& $\imeasure\left(\iset(X)\cup\iset(Y)\right)$ \\ \hline
		$H(X|Y)$	& $\imeasure\left(\iset(X)\cap \isetc(Y)\right)$ \\ \hline
		$I_H(X;Y)$ & $\imeasure\left(\iset(X)\cap \iset(Y)\right)$  \\ \hline
		$I_H(X;Y|Z)$	&  $\imeasure\left(\iset(X)\cap \iset(Y) \cap \isetc(Z)\right)$ \\ \hline
	\end{tabular}
	\vspace{0.2cm}
	\caption{Information measures and associated I-measure}
	\label{tab:ImeasureSummary}
\end{table}
Yeung's I-measure Theorem establishes the one-to-one correspondence between Shannon's information measures and $\imeasure$.
\begin{citedtheorem}{\cite{DBLP:journals/tit/Yeung91,Yeung:2008:ITN:1457455}} \label{thm:YeungUniqueness}\e{[I-Measure Theorem]}
	$\imeasure$ is the unique signed measure on $\mathcal{F}_n$ which is consistent with all Shannon's information measures (i.e., entropies, conditional entropies, and mutual information). 	
\end{citedtheorem}

Let $\sigma=(X;Y|Z)$. We denote by $\iset(\sigma)\eqdef\iset(X)\cap\iset(Y)\cap\isetc(Z)$ the set associated with $\sigma$ (see Table~\ref{tab:ImeasureSummary}). For a set of triples $\Sigma$, we define:
\begin{equation}
	\label{eq:SigmaSet}
\iset(\Sigma)\eqdef\bigcup_{\sigma \in \Sigma}\iset(\sigma)
\end{equation}
\begin{example}
Let $A$, $B$, and $C$ be three disjoint sets of RVs defined as follows: $A{=}A_1A_2A_3$, $B{=}B_1B_2$ and $C{=}C_1C_2$. Then, by Theorem~\ref{thm:YeungUniqueness}: $H(A){=}\mu^*(\iset(A)){=}\mu^*(\iset(A_1){\cup}\iset(A_2){\cup}\iset(A_3))$, $H(B){=}\mu^*(\iset(B)){=}\mu^*(\iset(B_1){\cup}\iset(B_2))$,~and $\mu^*(\isetc(C)){=}\mu^*(\isetc(C_1){\cap} \isetc(C_2))$. By Table~\ref{tab:ImeasureSummary}:  $I(A;B|C){=}\mu^*(\iset(A)\cap \iset(B)\cap \isetc(C))$.
\end{example}
We denote by $\Delta_n$ the set of signed measures $\mu^*:\mathcal{F}_n \rightarrow \real_{\geq 0}$ that assign non-negative values to the atoms $\mathcal{F}_n$. We call these \e{positive I-measures}.
\begin{citedtheorem}{\cite{Yeung:2008:ITN:1457455}}\label{thm:YeungBuildMeasure}
	If there is no constraint on $X_1,\dots,X_n$, then $\imeasure$ can take any set of nonnegative values on the nonempty atoms of $\mathcal{F}_n$.
\end{citedtheorem}	
Theorem~\ref{thm:YeungBuildMeasure} implies that every positive I-measure $\mu^*$ corresponds to a function that is consistent with the Shannon inequalities, and is thus a polymatroid. Hence, $\Delta_n\subset \Gamma_n$ is the set of polymatroids with a positive I-measure that we call \e{positive polymatroids}.

\subsection{Bayesian Networks}
A Bayesian network encodes the CIs of a probability distribution using a
Directed Acyclic Graph (DAG). Each node $X_i$ in a Bayesian network corresponds
to the variable $X_i\in \Omega$, a set of nodes $\alpha$ correspond to the set of variables $X_\alpha$, and $x_i \in \D_i$ is a
value from the domain of $X_i$. Each node $X_i$ in the network represents the distribution $p(X_i \mid X_{\pi(i)})$ where $X_{\pi(i)}$ is a set of variables that
correspond to the parent nodes $\pi(i)$ of $i$. The distribution represented by a Bayesian network is
\begin{equation}\label{eq:BN}
	p(x_1,\dots,x_n)=\prod_{i=1}^np(x_i| x_{\pi(i)})
\end{equation}
(when $i$ has no parents then $X_{\pi(i)}=\emptyset$).

Equation~\ref{eq:BN} implicitly encodes a set of $n$ conditional independence statements, called the \e{recursive basis} for the network:
\begin{equation}
	\label{eq:recursiveSet}
	\Sigma\eqdef\set{(X_i; X_1\dots X_{i-1}{\setminus} \pi(X_i) \mid \pi(X_i)): i\in [n]}
\end{equation}
The implication problem associated with Bayesian Networks is to determine whether $\Sigma \implies \tau$ for a CI $\tau$.
Geiger and Pearl have shown that $\Sigma \implies \tau$ iff $\tau$ can be derived from $\Sigma$ using the \e{semigraphoid axioms}~\cite{DBLP:journals/networks/GeigerVP90}.
Their result establishes that the semigraphoid axioms are  sound and complete for inferring CI statements from the recursive basis.
\eat{
 Since the semigraphoid axioms follow from the Shannon inequalities, this result says that the Shannon inequalities are both sound and complete for inferring CI statements from the recursive basis.
}

\section{The Relaxation Problem}
\label{sec:problem:def}

We now formally define the relaxation problem.  We fix a set of
variables $\Omega = \set{X_1, \ldots, X_n}$, and consider triples of
the form $\sigma = (Y;Z|X)$, where $X,Y,Z \subseteq \Omega$, which we
call a \e{conditional independence}, CI.  An \e{implication} is a
formula $\Sigma \implies \impliedCI$, where $\Sigma$ is a set of CIs
called \e{antecedents} and $\tau$ is a CI called \e{consequent}.  For
a CI $\sigma=(Y;Z|X)$, we define $h(\sigma)\eqdef I_h(Y;Z|X)$, for a
set of CIs $\Sigma$, we define
$h(\Sigma)\eqdef\sum_{\sigma \in \Sigma}h(\sigma)$.  Fix a set $K$
s.t.  $K \subseteq \Gamma_n$.

\begin{defn} \label{def:ei:ai} The \e{exact implication} (EI)
	$\Sigma \implies \impliedCI$ holds in $K$, denoted
	$K \models_{EI} (\Sigma \implies \impliedCI)$ if, forall $h \in K$,
	$h(\Sigma)=0$ implies $h(\tau)=0$.  The \e{$\lambda$-approximate
		implication} ($\lambda$-AI) holds in $K$, 
	in notation
	$K \models\lambda\cdot h(\Sigma) \geq h(\tau)$, if
	$\forall h \in K$, $\lambda\cdot h(\Sigma) \geq h(\tau)$. The
	\e{approximate implication} holds, in notation
	$K \models_{AI} (\Sigma \implies \impliedCI)$, if there exist a finite
	$\lambda \geq 0$ such that the $\lambda$-AI holds.	
\end{defn}
\eat{
We will sometimes consider an equivalent definition for AI, as
$\sum_{\sigma \in \Sigma} \lambda_\sigma h(\sigma) \geq h(\tau)$,
where $\lambda_\sigma \geq 0$ are coefficients, one for each
$\sigma \in \Sigma$; these two definitions are equivalent, by taking
$\lambda = \max_\sigma \lambda_\sigma$.}Notice that both exact (EI) and approximate (AI) implications
are preserved under subsets of $K$: if
$K_1 {\subseteq} K_2$ and $K_2 {\models_x} (\Sigma {\implies} \impliedCI)$, then
$K_1 {\models_x} (\Sigma {\implies} \impliedCI)$, for
$x {\in} \set{EI,AI}$.

Approximate implication always implies its exact counterpart.  Indeed, if $h(\tau) \leq \lambda \cdot h(\Sigma)$
and $h(\Sigma)=0$, then $h(\tau)\leq 0$, which further implies that
$h(\tau)=0$, because $h(\tau)\geq 0$ for every triple $\tau$, and every
polymatroid $h$.  In this paper we study the reverse.

\begin{defn}
	Let $\mathcal{L}$ be a syntactically-defined class of implication
	statements $(\Sigma \Rightarrow \tau)$, and let
	$K \subseteq \Gamma_n$.  We say that $\mathcal{L}$ \e{admits a
		$\lambda$-relaxation} in $K$, if every exact implication statement
	$(\Sigma \Rightarrow \impliedCI)$ in $\mathcal{L}$ has a $\lambda$-approximation:
	$$K\models_{EI} \Sigma \Rightarrow \impliedCI \text{ iff }
	K \models_{AI} \lambda \cdot h(\Sigma) \geq h(\impliedCI).$$\eat{  We say that
	$\mathcal{I}$ admits a \e{$\lambda$-relaxation} if every EI admits a
	$\lambda$-AI.}
\end{defn}
In this paper, we focus on $\lambda$-relaxation in the set $\Gamma_n$ of polymatroids, and two syntactically-defined classes: 1) Where $\Sigma$ is the recursive basis of a Bayesian network (see~\eqref{eq:recursiveSet}), and 2) Where $\Sigma$ is a set of marginal CIs.

\begin{example} 
	Let $\Sigma{=} \set{(A;B|\emptyset),(A;C|B)}$,~and
	$\impliedCI{=}(A;C|\emptyset)$. 
	Since $I_h(A;C|\emptyset) {\leq}I_h(A;BC)$, and since $I_h(A;BC){=}I_h(A;B|\emptyset) {+} I_h(A;C|B)$ by the chain rule~\eqref{eq:ChainRuleMI}, then the exact implication $\entropicPlhdrl_n \models_{EI}\Sigma\implies \tau$ admits an AI with $\lambda=1$ (i.e., a $1$-$AI$).
	\eat{
		Then both EI and AI hold forall
		polymatroids:
		\begin{align*}
			& \text{EI} && I_h(A;B|\emptyset)=0\wedge I_h(A;C|B)=0 \Rightarrow I_h(A;C|\emptyset)=0\\
			& \text{AI} && I_h(A;C|\emptyset) \leq I_h(A;B|\emptyset) + I_h(A;C|B)
		\end{align*}
	}
\end{example}

\section{Formal STATEMENT OF RESULTS}
\label{sec:results}
We generalize the results of Geiger et al.~\cite{DBLP:conf/uai/GeigerVP89,GEIGER1991128}, by proving that implicates $\tau{=}(X;Y|Z)$ of the recursive set~\cite{DBLP:conf/uai/GeigerVP89}, and of marginal CIs~\cite{GEIGER1991128}, admit a $1$, and $|X||Y|$-approximation respectively, and thus continue to hold also approximately.
\eat{
For the exact implication from a set of marginal CIs~\cite{GEIGER1991128}, we further show that our result applies to a \e{any} implicate, and is not limited to marginal implicates.

Geiger and Pearl in three dimensions: by extending them to the approximate case, broadening the set of implicates (for the case of marginal antecedents), and broadening the set of antecendents (for the case of recursive constraints).
}
\subsection{Implication From Recursive CIs}
\eat{We show that the implication from a recursive set of CIs (see~\eqref{eq:recursiveSet}) and functional dependencies admits a 1-relaxation.}
 Geiger et al.~\cite{DBLP:conf/uai/GeigerVP89} prove that the semigraphoid axioms are sound and complete for the implication from the recursive set (see~\eqref{eq:recursiveSet}). They further showed that the set of implicates can be read off the appropriate DAG via the d-separation procedure.
We show that every such exact implication can be relaxed, admitting a $1$-relaxation, guaranteeing a bounded approximation for the implicates (CI relations) read off the DAG by d-separation. 

We recall the definition of the recursive basis $\Sigma$ from~\eqref{eq:recursiveSet}:
\begin{equation}
	\label{eq:nRecursiveSet}
	\Sigma \eqdef \set{(X_i;R_i|B_i) : i\in [1,n], R_iB_i=U^{(i)}}
\end{equation}
where $B_i{\eqdef}\pi(X_i)$ and $U^{(i)}{\eqdef} \set{X_1,\dots,X_{i-1}}$.
\eat{
\begin{defn}[Recursive Set]
\label{def:enhancedRecursiveSet}
Let $\Omega=\set{X_1,\dots,X_n}$ denote a set of ordered RVs, and for every $i\in [1,n]$ we let $U^{(i)}\eqdef \set{X_1,\dots,X_{i-1}}$, and $R_i=\pi(X_i)$. The recursive set is defined:
\begin{equation}
	\Sigma \eqdef \set{(X_i;R_i|B_i) | i\in [1,n], R_iB_i=U^{(i)}}
\end{equation}
\end{defn}
}
We observe that $|\Sigma|{=}n$, there is a single triple $\sigma_n{=}(X_n;R_n|B_n){\in} \Sigma$ that mentions $X_n$, and that $\sigma_n$ is saturated.

We recall that $\Delta_n\subset \Gamma_n$ is the set of polymatroids whose I-measure assigns non-negative values to the atoms $\mathcal{F}_n$ (see Section~\ref{sec:imeasure}).
\begin{theorem}
	\label{thm:recursiveCIs}
	Let $\Sigma$ be a recursive set of CIs (see~\eqref{eq:nRecursiveSet}), and let $\tau = (A;B|C)$. Then the following holds: 
	\begin{align}
	\Delta_n \models_{EI} \Sigma \Rightarrow \tau& &\text{ iff } & &\Gamma_n \models h(\Sigma) \geq h(\tau)
	\end{align}
\end{theorem}
We note that the only-if direction of Theorem~\ref{thm:recursiveCIs} is immediate, and follows from the non-negativity of Shannon's information measures. We prove the other direction in Section~\ref{sec:recursiveProof}.
Theorem~\ref{thm:recursiveCIs} states that it is enough that the exact implication holds on all of the positive polymatroids $\Delta_n$, because this implies the (even stronger!) statement
$\Gamma_n \models h(\Sigma) \geq h(\tau)$.
\eat{
 Observe that proving the claim only requires us to assume that the exact implication holds in the positive subset of polymatroids $\Delta_n \subset \Gamma_n$.
}
\subsection{Implication from Marginal CIs}
We show that \e{any} implicate $\tau{=}(A;B|C)$ from a set of marginal CIs has an $|A|{\cdot}|B|$-approximation. This generalizes the result of Geiger, Paz, and Pearl~\cite{GEIGER1991128}, which proved that the semigraphoid axioms are sound and complete for deriving marginal CIs. \eat{, while we show bounded relaxation, and thus implication, for \e{any} implicate of marginal CIs.}
\eat{ That is, when $\Sigma$ is a set of marginal CIs (e.g., $(X;Y)$), and $\tau$ is \e{any} (not necessarily marginal) CI, then $\Sigma \implies \tau$ implies that $h(\tau) \leq \frac{n^2}{4}h(\Sigma)$. }
\begin{thm}
	\label{thm:marginal}
	Let $\Sigma$ be a set of marginal CIs, and $\tau=(A;B|C)$ be any CI.
	\begin{align}
		\Gamma_n \models_{EI} \Sigma \Rightarrow \tau& &\text{ iff }& &\Gamma_n \models   (|A||B|)h(\Sigma) \geq h(\tau) \eat{\frac{n^2}{4}h(\Sigma) \geq h(\tau)}
	\end{align}
\end{thm}
Also here, the only-if direction of Theorem~\ref{thm:marginal} is immediate\eat{, and follows from the positivity of Shannon's information measures.}, and we prove the other direction in Section~\ref{sec:marginalProof}. 
\eat{
To prove the claim, we need to assume that the exact implication holds for the entire set of polymatroids $\Gamma_n$ (as opposed to only the subset $\Delta_n$ in Theorem~\ref{thm:recursiveCIs}). 
}
\eat{
In the proof, we make use of the \e{covering property} of EI in the set of polymatroids $\Gamma_n$, described in Lemma~\ref{lem:implicationCovers}.
}
\eat{
We observe that to prove the $\lambda$-approximation for the enhanced recursive set of CIs (Definition~\ref{def:enhancedRecursiveSet}), it is enough to assume that the exact implication holds only for the set $\Delta_n\subset \Gamma_n$ of positive polymatroids (Theorem~\ref{thm:recursiveCIs}).
This assumption is not enough for proving $\lambda$-approximation in the case of marginal CIs, where we need to assume that the implication holds for \e{all} polymatroids $\Gamma_n$ (Theorem~\ref{thm:marginal}). 
}

\section{Properties of Exact Implication}
\label{sec:lemmas}
In this section, we use the I-measure to characterize some general properties of exact implication in the set of positive polymatroids $\Delta_n$ (Section~\ref{sec:implicationPositivePolymatroids}), and the entire set of polymatroids $\Gamma_n$ (Section~\ref{sec:EIPolymatroids}). The lemmas in this section will be used for proving the approximate implication guarantees presented in Section~\ref{sec:results}.

In what follows, $\Omega=\set{X_1,\dots,X_n}$ is a set of $n$ RVs, $\Sigma$ denotes a set of triples $(A;B|C)$ representing mutual information terms, and $\tau$ denotes a single triple. We denote by $\vars(\sigma)$ the set of RVs mentioned in $\sigma$ (e.g., if $\sigma=(X_1X_2;X_3|X_5)$ then $\vars(\sigma)=X_1\dots X_5$).

\subsection{Exact implication in the set of positive polymatroids}
\label{sec:implicationPositivePolymatroids}
\def\implicationInclusion{
\eat{	Let $\Sigma$ denote a set of mutual information terms, and $\tau$ a single mutual information term. Then:
}
The following holds:
	\begin{align*}
	\eat{	 \text{if  } && \Gamma_n \models_{EI} \Sigma \implies \tau && \text{ then }&& \iset(\Sigma) \supseteq \iset(\tau) \\}
		&&\Delta_n \models_{EI} \Sigma \implies \tau&& \text{ iff }&& \iset(\Sigma) \supseteq \iset(\tau) 		
	\end{align*}
}
\begin{lemma}\label{lem:implicationInclusion}
	\implicationInclusion
\end{lemma}
\begin{proof}
	\eat{
	Suppose that $\Delta_n \models_{EI} \Sigma \implies \tau$, and assume, by contradiction,} Suppose that $\iset(\impliedCI) {\not\subseteq} \iset(\Sigma)$, and let $b\in \iset(\impliedCI){\setminus}\iset(\Sigma)$. By Theorem~\ref{thm:YeungBuildMeasure} there exists a positive polymatroid in $\Delta_n$ with an $I$-measure $\imeasure$ that takes the following non-negative values on its atoms: $\imeasure(b){=}1$, and $\imeasure(a)=0$ for any atom $a{\in}\mathcal{F}_n$ where $a\neq b$.\eat{
	\begin{equation*}
		\imeasure(a)=
		\begin{cases}
			1 & \text{if } a=b \\
			0 & \text{otherwise}
		\end{cases}
	\end{equation*}
	}
Since $b\notin \iset(\Sigma)$, then $\imeasure(\Sigma)=0$ while $\imeasure(\impliedCI)= 1$. Hence, $\Delta_n {\not\models}\Sigma \implies \tau$.\eat{ which contradicts the exact implication.}
	
	Now, suppose that $\iset(\Sigma) {\supseteq} \iset(\tau)$.
	Then for any positive I-measure $\imeasure{:} \mathcal{F}_n {\rightarrow} \real_{\geq 0}$, we have that $\imeasure(\iset(\Sigma)) {\geq} \imeasure(\iset(\tau))$. By Theorem~\ref{thm:YeungUniqueness}, $\mu^*$ is the unique signed measure on $\mathcal{F}_n$ that is consistent with all of Shannon's information measures. Therefore, $h(\Sigma) {\geq} h(\tau)$. The result follows from the non-negativity of the Shannon information measures.
\end{proof}
An immediate consequence of Lemma~\ref{lem:implicationInclusion} is that $\iset(\Sigma) {\supseteq} \iset(\tau)$ is a necessary condition for implication between polymatroids. 
\begin{corr}
	\label{corr:inclusionGamman}
If $\Gamma_n \models_{EI} \Sigma \implies \tau$ then $\iset(\Sigma) \supseteq \iset(\tau)$.
\end{corr}
\begin{proof}
If $\Gamma_n \models_{EI} \Sigma \implies \tau$ then it must hold for any subset of polymatroids, and in particular, $\Delta_n \models_{EI} \Sigma \implies \tau$. The result follows from Lemma~\ref{lem:implicationInclusion}.
\end{proof}
\eat{
In what follows, we recall the definition of $\iset(\sigma)$ from Table~\ref{tab:ImeasureSummary}.
}
\begin{lemma}
	\label{lem:excludeSigma}
	Let $\Delta_n \models_{EI} \Sigma \implies \tau$, and let $\sigma \in \Sigma$ such that $\iset(\sigma)\cap \iset(\tau)=\emptyset$. Then $\Delta_n \models_{EI} \Sigma{\setminus}\set{\sigma} \implies \tau$.
\end{lemma}
\begin{proof}
	Let $\Sigma'=\Sigma {\setminus} \set{\sigma}$, and suppose that $\Delta_n \not\models_{EI} \Sigma' \implies \tau$. By Lemma~\ref{lem:implicationInclusion}, we have that $\iset(\Sigma')\not\supseteq \iset(\tau)$. In other words, there is an atom $a\in \mathcal{F}_n$ such that $a\in \iset(\tau){\setminus} \iset(\Sigma')$.
	In particular, $a \notin \iset(\sigma) \cup \iset(\Sigma')= \iset(\Sigma)$. Hence, $\iset(\tau) \not\subseteq \iset(\Sigma)$, and by Lemma~\ref{lem:implicationInclusion} we get that $\Delta_n \not\models_{EI} \Sigma \implies \tau$.
\end{proof}

\subsection{Exact Implication in the set of polymatroids}
\label{sec:EIPolymatroids}
The main technical result of this section is Lemma~\ref{lem:implicationCovers}. We start with two short technical lemmas.

\def\chainRuleTechnicalLemma{ Let $\sigma = (A;B|C)$ and
	$\tau = (X;Y|Z)$ be CIs such that $X\subseteq A$, $Y\subseteq B$,
	$C \subseteq Z$ and $Z \subseteq ABC$. Then,
	$\entropicPlhdrl_n \models h(\tau) \leq h(\sigma)$.  }

\begin{lemma}\label{lem:chainRuleTechnicalLemma}
	\chainRuleTechnicalLemma
\end{lemma}
\begin{proof}
	Since $Z {\subseteq} ABC$, we denote by $Z_A{=}A{\cap}Z$, $Z_B{=}B{\cap} Z$, and $Z_C{=}C {\cap} Z$. Also, denote by $A'{=}A{\setminus} (Z_A{\cup} X)$, $B'{=}B{\setminus} (Z_B{\cup} Y)$. So, we have that: $I(A;B|C){=}I(Z_AA'X;Z_BB'Y|C)$.
	By the chain rule, we have that:
	\begin{align*}
		&I(Z_AA'X;Z_BB'Y|C)=\\
		&I(Z_A;Z_B|C)+I(A'X;Z_B|CZ_A)\\
		&+I(Z_A;B'Y|Z_BC)	+\boldsymbol{I(X;Y|CZ_AZ_B)}\\
		&+I(X;B'|CZ_AZ_BY)+I(A';B'Y|CZ_AZ_BX)
	\end{align*}
	Noting that $Z=CZ_AZ_B$, we get that $I(X;Y|Z)\leq I(A;B|C)$ as required.
\end{proof}

\def\lessVarsLemma{
	Let $\Sigma=\set{\sigma_1,\dots,}$ be a set of triples such that $\vars(\sigma_i)\subseteq \set{X_1,\dots,X_{n-1}}$ for all $\sigma_i \in \Sigma$. Likewise, let $\tau$ be a triple such that $\vars(\tau)\subseteq \set{X_1,\dots,X_{n-1}}$. Then:
\eat{	If $\Sigma$ and $\tau$ are defined over RVs $\set{X_1,\dots,X_{n-1}}$ then:}
	\begin{align}
		\label{eq:lessVarsLemma}
		\Gamma_n \models_{EI} \Sigma \implies \tau && \text{ iff } && \Gamma_{n-1} \models_{EI} \Sigma \implies \tau
	\end{align}
}
\begin{lemma}
	\label{lem:lessVarsLemma}
	\lessVarsLemma
\end{lemma}
\begin{proof}
Suppose that $\Gamma_n \not\models_{EI} \Sigma \implies \tau$. Then there exists a polymatroid (Section~\ref{subsec:information:theory}) $f:2^{[n]}\rightarrow \real$ such that $f(\sigma)=0$ for all $\sigma \in \Sigma$, and $f(\tau)\neq 0$.
We define $g:2^{[n-1]}\rightarrow \real$ as follows:
\begin{align}
\label{eq:nminus1Proof}
g(A)=f(A) && \text{ for all }&& A\subseteq \set{X_1,\dots,X_{n-1}}
\end{align}
Since $f$ is a polymatroid, then so is $g$. Further, since $\Sigma$ does not mention $X_n$ then, by~\eqref{eq:nminus1Proof}, we have that $g(\sigma)=f(\sigma)$ for all $\sigma \in \Sigma$. Hence, $\Gamma_{n-1}\not\models_{EI} \Sigma \implies \tau$.

If $\Gamma_{n-1} \not\models_{EI} \Sigma \implies \tau$. Then there exists a polymatroid  $g:2^{[n-1]}\rightarrow \real$ such that $g(\sigma)=0$
for all $\sigma \in \Sigma$, and $g(\tau)\neq 0$.
Define $f:2^{[n]}\rightarrow \real$ as follows:
\begin{align}
	\label{eq:nminus1Proof2}
	f(A)=g(A\setminus X_n) && \text{ for all }&& A\subseteq \set{X_1,\dots,X_{n}}
\end{align}
We claim that $f\in \Gamma_n$ (i.e., $f$ is a polymatroid). 
It then follows that $\Gamma_n\not\models \Sigma \implies \tau$ because by the assumption that $\vars(\Sigma)$ and $\vars(\tau)$ are subsets of $\set{X_1,\dots,X_{n-1}}$, then $f(\sigma)=g(\sigma)$ for all $\sigma \in \Sigma$. Hence, $f(\Sigma)=g(\Sigma)=0$ while $f(\tau)=g(\tau)\neq 0$.

We now prove the claim.
First, by~\eqref{eq:nminus1Proof2}, we have that $f(\emptyset)=g(\emptyset)=0$. We show that $f$ is monotonic. So let $A\subseteq B \subseteq \set{X_1,\dots,X_n}$. If $X_n\notin B$ then $X_n \notin A$ and we have that:
\[
f(B)-f(A)=g(B)-g(A)\underbrace{\geq}_{\substack{B\supseteq A\\ g\in \Gamma_{n-1}}} 0
\]
If $X_n \in B\setminus A$ then we let $B=B'X_n$, and we have:
\[
f(B'X_n)-f(A)\underbrace{=}_{\eqref{eq:nminus1Proof2}}g(B')-g(A)\underbrace{\geq}_{B'\supseteq A} 0
\]
Finally, if $X_n\in A\subseteq B$, then by letting $B=B'X_n$, $A=A'X_n$, we have that:
\[
f(B'X_n)-f(A'X_n)\underbrace{=}_{\eqref{eq:nminus1Proof2}}g(B')-g(A')\geq 0
\]

We now show that $f$ is submodular. Let $A,B \subseteq \set{X_1,\dots,X_n}$. If $X_n\notin A\cup B$ then $f(Y)=g(Y)$ for every set $Y{\in}\set{A,B,A{\cup}B,A{\cap}B}$. Since $g$ is submodular, then $f(A){+}f(B){\geq}f(A{\cup}B){+}f(A{\cap}B)$.
If $X_n\in A\setminus B$ then we write $A=A'X_n$ and observe that, by~\eqref{eq:nminus1Proof2}: $f(A'X_n)=g(A')$, $f(A{\cup}B)=f(A'X_n{\cup}B)=g(A'{\cup}B)$, that $f(B)=g(B)$, and that $f(A{\cap}B)=f(A'{\cap}B)=g(A'{\cap}B)$. 
Hence:
$f(A)+f(B)=g(A')+g(B)\geq g(A'{\cup}B)+g(A'{\cap}B)$.
The case where $X_n\in B{\setminus}A$ is symmetrical.
Finally, if $X_n\in A\cap B$ then $X_n\in Y$ for all $Y\in \set{A,B,A{\cap}B,A{\cup}B}$. Hence, for every $Y$ in this set, we write $Y=Y'X_n$. In particular, by~\eqref{eq:nminus1Proof2} we have that $f(Y)=f(Y'X_n)=g(Y')$, and the claim follows since $g\in \Gamma_n$.
\end{proof}
\eat{
\begin{proof}
	Suppose that $\Gamma_{n} \models \Sigma \not\implies \tau$. Then there exists an $n$-variable distribution $p_\tau(X_1,\dots,X_n)$ that fulfills all CI statements in $\Sigma$ (i.e., $h(\sigma)=0$ for all $\sigma\in \Sigma$), and $h(\tau)>0$. Define $g_\tau(X_1,\dots,X_{n-1})=\sum_{y\in \D_n}p_\tau(X_1,\dots,X_{n-1},y)$. Since $\Sigma$ and $\tau$ do not mention $X_n$ then $h_{p_\tau}(\Sigma)=h_{g_\tau}(\Sigma)$ and likewise $h_{p_\tau}(\tau)=h_{g_\tau}(\tau)$ (see~\eqref{eq:entropy}). Hence, $\Gamma_{n-1} \models \Sigma \not\implies \tau$.
	
	If $\Gamma_{n-1} \not\models \Sigma \implies \tau$ then there exists an $n-1$-variable probability distribution $g_\tau(X_1,\dots,X_{n-1})$ that fulfills all CI statements in $\Sigma$, and $h(\tau)>0$. Define 
	$$p_\tau(X_1,\dots,X_{n})=\begin{cases}
		g_\tau(X_1,\dots,X_{n-1}) & X_n=0 \\
		0 & \text{otherwise}
	\end{cases}$$
	Since $\Sigma$ and $\tau$ do not mention $X_n$ then $h_{p_\tau}(\Sigma)=h_{g_\tau}(\Sigma)$ and likewise $h_{p_\tau}(\tau)=h_{g_\tau}(\tau)$. Hence, $\Gamma_{n} \models \Sigma \not\implies \tau$.
\end{proof}
}

\begin{lemma}
	\label{lem:implicationCovers}
	Let $\tau=(A;B|C)$.	If $\Gamma_n \models_{EI} \Sigma \implies \tau$ then there exists a triple $\sigma=(X;Y|Z) \in \Sigma$ such that:
	\begin{enumerate}
		\item 	$XYZ \supseteq ABC$, and
		\item $ABC \cap X \neq \emptyset$ and $ABC \cap Y \neq \emptyset$.
	\end{enumerate}

\end{lemma}
\begin{proof}
	Let $\tau=(A;B|C)$, where $A=a_1\dots a_m$, $B=b_1\dots b_\ell$, $C=c_1\dots c_k$, and $U=\Omega{\setminus}ABC$. Following~\cite{DBLP:journals/iandc/GeigerPP91}, we construct the parity  distribution $P(\Omega)$ as follows. 
	We let all the RVs, except $a_1$, be independent binary RVs with probability $\frac{1}{2}$ for each of their two values, and let $a_1$ be determined from $ABC\setminus \set{a_1}$ as follows:
	\begin{equation}
		\label{eq:parity}
		a_1 = \sum_{i=2}^m a_i + \sum_{i=1}^\ell b_i + \sum_{i=1}^k c_i \pmod{2}
	\end{equation}
	Let $D \subseteq \Omega$ and $\bd \in \D(D)$. We denote by $D_{ABC}=D \cap ABC$, and by $\bd_{ABC}$ the assignment $\bd$ restricted to the RVs $D_{ABC}$.
	We show that if $D_{ABC}\subsetneq ABC$ then the RVs in $D$ are pairwise independent.\eat{
	Now, let $D \subseteq \Omega$, such that $D \cap \vars(\tau) \subsetneq ABC$. We show that the RVs in $D$ are pairwise independent. Let $\bd \in \D(D)$, and we denote $D{\cap} ABC$ by $D_{ABC}$, and by $\bd_{ABC}$ the assignment $\bd$ restricted to the RVs $D_{ABC}$.}
	By the definition of $P$ we have that:
	\begin{align*}
		P(D=\bd)=\left(\frac{1}{2}\right)^{|D{\cap}U|}P(D_{ABC}{=}\bd_{ABC})
	\end{align*}
	There are two cases with respect to $D$. If $a_1\notin D$ then, by definition, $P(D_{ ABC}{=}\bd_{ABC})=\left(\frac{1}{2}\right)^{|D_{ABC}|}$, and overall we get that $P(D{=}\bd){=}\left(\frac{1}{2}\right)^{|D|}$. Hence, the RVs in $D$ are pairwise independent.
	If $a_1{\in}D$, then since $D_{ABC}\subsetneq ABC$ it holds that $P(a_1|D_{ABC}{\setminus}\set{a_1}){=}P(a_1)$. To see this, observe that:
	\begin{align*}
		&P(a_1{=}1|D_{ABC}{\setminus}\set{a_1}) \nonumber \\
		&=\begin{cases}
			\frac{1}{2} & \text{if } \sum_{y\in D_{ABC}{\setminus}\set{a_1}}y \pmod{2}{=}0 \\
			\frac{1}{2} & \text{if }\sum_{y\in D_{ABC}{\setminus}\set{a_1}}y \pmod{2}{=}1 
		\end{cases}
	\end{align*}
	because if, w.l.o.g, $\sum_{y\in D_{ABC}{\setminus}\set{a_1}}y \pmod{2}{=}0$, then $a_1{=}1$ implies that $\sum_{y\in ABC{\setminus}D}y \pmod{2}{=}1$, and this is the case for precisely half of the assignments $ABC{\setminus}D{\rightarrow}\set{0,1}^{|ABC{\setminus}D|}$. 
	Hence, for any $D \subseteq \Omega$ such that $D \cap \vars(\tau) \subsetneq ABC$ it holds that $P(D{=}\bd)=\prod_{y \in D}P(y{=}\bd_y)=\left(\frac{1}{2}\right)^{|D|}$, and therefore the RVs are pairwise independent.
	
	By definition of entropy (see~\eqref{eq:entropy}) we have that $H(X_i)=1$ for every binary RV in $\Omega$. Since the RVs in $D$ are pairwise independent then $H(D)=\sum_{y\in D}H(y)=|D|$\footnote{This is due to the chain rule of entropy, and the fact that if $X$ and $Y$ are independent RVs then $H(Y|X)=H(Y)$.}. Furthermore, for any $(X;Y|Z)\in \Sigma$ s.t. $XYZ \not\subseteq ABC$ we have that:
	\begin{align*}
	I(X;Y|Z)&=H(XZ)+H(YZ)-H(Z)-H(XYZ)\\
	        &=|XZ|+|YZ|-|Z|-|XYZ| \\
	        &=|X|+|Y|+|Z|-|XYZ| \\
	        &=0
	\end{align*}	
	On the other hand, letting $A'{\eqdef}A{\setminus}\set{a_1}$, then by chain rule for entropies, and noting that, by~\eqref{eq:parity}, $ABC{\setminus}a_1 \fd a_1$, then:
	\begin{align*}
		H(\vars(\tau))=H(ABC)&=H(a_1A'BC) \\
		    &=H(a_1|A'BC)+H(A'BC)\\
		    &=0+|ABC|-1=|ABC|-1.
	\end{align*}
	and thus
	\begin{align}
		I(A;B|C)&=H(AC)+H(BC)-H(C)-H(ABC) \nonumber \\
		    &=|AC|+|BC|-|C|-(|ABC|-1) \label{eq:tau1}\\
		    &=1 \nonumber
	\end{align}
	In other words, the parity distribution $P$ of~\eqref{eq:parity} has an entropic function $h_P\in \Gamma_n$, such that $h_P(\sigma)=0$ for all $\sigma \in \Sigma$ where $\vars(\sigma)\not\supseteq ABC$, while $h_P(\tau)=1$. Hence, if $\Gamma_n \models \Sigma \implies \tau$, then there must be a triple $\sigma=(X;Y|Z) \in \Sigma$ such that $XYZ \supseteq ABC$. \eat{and we arrive at a contradiction to the assumption $\Gamma_n \models_{EI} \Sigma \implies \tau$.}
	
	Now, suppose that $ABC\subseteq XYZ$ and that $ABC \cap Y=\emptyset$. In other words, $ABC\subseteq XZ$. We denote $X_{ABC}\eqdef X\cap ABC$ and $Z_{ABC}=Z\cap ABC$. Therefore, we can write $I(X;Y|Z)$ as $I(X_{ABC}X';Y|Z_{ABC}Z')$ where $X'=X{\setminus}X_{ABC}$ and $Z'=Z{\setminus}Z_{ABC}$. 
	It is easily shown that if $ABC \subseteq X$ or $ABC \subseteq Z$ then $I(X;Y|Z)=0$.	
	Otherwise (i.e., $X_{ABC}\neq \emptyset$ and $Z_{ABC}\neq \emptyset$), then due to the properties of the parity function, we have that $H(YZ'Z_{ABC})=H(Y)+H(Z')+H(Z_{ABC})$. Noting that $X_{ABC}Z_{ABC}=ABC$, we get that $I(X_{ABC}X';Y|Z_{ABC}Z')=0$.

	Overall, we showed that for all triples $(X;Y|Z)\in \Sigma$ that do not meet the conditions of the lemma, it holds that $I_{h_P}(X;Y|Z)=0$, while $I_{h_P}(A;B|C)=1$ (see~\eqref{eq:tau1}) where $h_P$ is the entropic function associated with the parity function $P$ in~\eqref{eq:parity}. Therefore, there must be a triple $\sigma\in \Sigma$ that meets the conditions of the lemma. Otherwise, we arrive at a contradiction to the EI.	
\end{proof}

\section{Approximate Implication for Recursive CIs}
\label{sec:recursiveProof}
We prove Theorem~\ref{thm:recursiveCIs}.
Let $P$ be a multivariate distribution over $\Omega{=}\set{X_1,\dots,X_n}$, and $\Sigma$ be a recursive set (see~\eqref{eq:nRecursiveSet}).
We prove Theorem~\ref{thm:recursiveCIs} by induction on the highest RV-index mentioned in any triple of $\Sigma$.

\eat{We prove Theorem~\ref{thm:recursiveCIs} by induction on $n$,
the number of variables in the distribution $P$ over $\Omega{=}\set{X_1,\dots,X_n}$. }
The claim trivially holds for $n{=}1$ (since no conditional independence statements are implied), so we assume correctness when the highest RV-index mentioned in $\Sigma$ is $\leq n{-}1$, and prove for $n$.

We recall that $\Sigma=\set{\sigma_1,\dots,\sigma_n}$ where $\sigma_i=(X_i;R_i|B_i)$ where $R_iB_i=\set{X_1,\dots,X_{i-1}}$. In particular, only $\sigma_n=(X_n;R_n|B_n)$ mentions the RV $X_n$, and it is saturated (i.e., $X_nR_nB_n=\Omega$).
We denote by $\Sigma'=\Sigma\setminus \set{\sigma_n}$, and note that $X_n\notin \vars(\Sigma')$.
The induction hypothesis states that:
\begin{align}	
	\Delta_{n} \models_{EI} \Sigma' \implies \tau && \text{iff} && \Gamma_{n} \models h(\Sigma') \geq h(\tau)  \label{eq:induction}
\end{align}
Equivalently, by Lemma~\ref{lem:implicationInclusion}, and due to the one-to-one correspondence between Shannon's information measures and $\imeasure$ (Theorem~\ref{thm:YeungUniqueness}), we can state the induction hypothesis:
\begin{align}	
	 \iset(\Sigma') \supseteq \iset(\tau) && \text{iff} && \imeasure(\iset(\Sigma'))\geq  \imeasure(\iset(\tau)) \label{eq:inductionImeasure}
\end{align}

Now, we consider $\tau=(X;Y|Z)$. We divide to three cases, and treat each one separately.
\begin{enumerate}[itemsep=0mm]
	\item $X_n \notin XYZ$
	\item $X_n \in Z$
	\item $X_n \in X$ (or, symmetrically, $X_n\in Y$)
\end{enumerate}

\paragraph{Case 1: $X_n \notin XYZ$.}
We will show that $\Delta_n \models_{EI} \Sigma' \implies \tau$, and the claim will follow from the induction hypothesis~\eqref{eq:induction} because $\Sigma'$ does not mention $X_n$, and
$h(\Sigma)\geq h(\Sigma')\geq h(\tau)$ as required.

Suppose, by way of contradiction, that $\Delta_n \models_{EI} \Sigma' \not\implies \tau$. Since neither $\Sigma'$ nor $\tau$ mention $X_n$ then,
by Lemma~\ref{lem:lessVarsLemma}, we have that $\Delta_{n-1} \models_{EI} \Sigma' \not\implies \tau$. 
Hence, by Lemma~\ref{lem:implicationInclusion}, we have that $\iset(\Sigma') \not\supseteq \iset(\tau)$, and
there exists an atom $a\in \mathcal{F}_{n-1}$ such that $a\in \iset(\tau){\setminus} \iset(\Sigma')$. Consequently, there exist two atoms $a_1,a_2 \in \mathcal{F}_n$ where:
\begin{align*}
	a_1\eqdef a \cap \iset(X_n) && a_2\eqdef a \cap \isetc(X_n)
\end{align*}
such that $\set{a_1,a_2}{\subseteq} \iset(\tau)$ and $\set{a_1,a_2}\cap \iset(\Sigma')=\emptyset$. 
\eat{
We define the positive I-measure $\imeasure: \mathcal{F}_n\rightarrow \real_{\geq 0}$ as follows.
\begin{equation}
	\imeasure(b)=\begin{cases}
		1 & b=a_2 \\
		0 & \text{otherwise}
	\end{cases}
\end{equation}
}
By our observation, $\sigma_n{=}(X_n;R|B)$. Therefore, we have that $\iset(\sigma_n) \subseteq \iset(X_n)$ (i.e., see Table~\ref{tab:ImeasureSummary}). So, we get that $a_2 {\notin} \iset(\sigma_n)$.
Overall, we have that $a_2 {\notin} \iset(\sigma_n){\cup} \iset(\Sigma'){=}\iset(\Sigma)$, and by Lemma~\ref{lem:implicationInclusion}, we get that $\Delta_n \models_{EI}\Sigma \not\implies \tau$, a contradiction.
\eat{Therefore, it must hold that $\Delta_n \models_{EI} \Sigma' \implies \tau$. }

\paragraph{Case 2: $\tau=(W;Y|ZX_n)$.}
Then $\iset(\tau){\subseteq}\isetc(X_n)$,
and since $\sigma_n$ has the form $\sigma_n{=}(X_n;R|B)$, then $\iset(\sigma_n){\subseteq}\iset(X_n)$ (see Table~\ref{tab:ImeasureSummary}).
Hence, $\iset(\tau)\cap \iset(\sigma_n)=\emptyset$, and by Lemma~\ref{lem:excludeSigma}, if $\Delta_n\models_{EI} \Sigma \implies \tau$ then it must hold that $\Delta_n\models_{EI} \Sigma'\implies \tau$, and the claim follows from the induction hypothesis~\eqref{eq:induction} because $\Sigma'$ does not mention $X_n$, and $h(\Sigma)\geq(\Sigma')\geq h(\tau)$.

\paragraph{Case 3: $\tau=(WX_n;Y|Z)$.}
By the chain rule (see~\eqref{eq:ChainRuleMI}):
\begin{equation}
	\label{eq:recursiveSetCase3}
(WX_n;Y|Z)=\underbrace{(W;Y|Z)}_{\tau_1}+\underbrace{(X_n;Y|WZ)}_{\tau_2}
\end{equation}
Hence, if $\Delta_n{\models}_{EI}\Sigma{\implies}\tau$ then $\Delta_n{\models}_{EI}\Sigma{\implies}\tau_1$, and $\Delta_n{\models}_{EI}\Sigma{\implies}\tau_2$.
We have already shown, in case 1, that the former implies $\Delta_n {\models_{EI}} \Sigma' {\implies} \tau_1$.\eat{, and by the induction hypothesis~\eqref{eq:induction} that $h(\Sigma'){\geq} h(\tau_1)$.
We now show that $h(\sigma_n){\geq}h(\tau_2)$, and this will prove the claim. }

Let $\sigma_n=(X_n;R|B)$, and let $Y=Y_1\dots Y_m$ where $m\geq 1$. We claim that $Y \subseteq R$. Since $\sigma_n$ is saturated then $Y \subseteq RB$. Now, suppose by way of contradiction, that $Y_i \in B$ for some $i\in [1,m]$.
Consider the atom 
$$a\eqdef\iset(X_n)\cap \iset(Y_i)\bigcap_{X\in [n]{\setminus}\set{X_n,Y_i}}\isetc(X).$$

We observe that $a\in \iset(\tau_2)$. Since, by our assumption $Y_i{\in}B$, then $a\notin \iset(\sigma_n)$.
On the other hand, for every $\sigma=(A;B|C)\in \Sigma'$, we also have that $a\notin \iset(\sigma)$. To see why, note that $X_n\notin AB$.  Therefore, every atom of $\iset(\sigma)$ contains at least two sets in positive form: $\iset(X_i)$ for some $X_i\in A$ and $\iset(X_j)$ for some $X_j\in B$. Since neither of these are $X_n$, then at least one of them appears in negative form in $a$.
Overall, we get that $\iset(X_n;Y|WZ)\not\subseteq \iset(\Sigma)$, and by Lemma~\ref{lem:implicationInclusion} that $\Delta_n\models_{EI} \Sigma \not\implies \tau_2$. Hence, from~\eqref{eq:recursiveSetCase3}, we get that $\Delta_n\models_{EI} \Sigma \not\implies \tau$, a contradiction.

Since $Y{\subseteq} R$, we can write $\sigma_n{=}(X_n;YR_WR_ZR'|B_WB_ZB')$ where $R_W\eqdef R\cap W$, $R_Z\eqdef R\cap Z$, and $R'\eqdef R{\setminus}R_WR_ZY$. Likewise, $B_W{\eqdef}B{\cap}W$,$B_Z{\eqdef}B{\cap}Z$, and $B'{\eqdef}B{\setminus}B_WB_Z$. Further, since $\sigma_n$ is saturated then $W=R_WB_W$ and $Z=R_ZB_Z$. By the chain rule, we have that:
{
\small{
\begin{align}
h(\sigma_n)&=I_h(X_n;YR_WR_ZR'|B_WB_ZB') \nonumber \\
&=I_h(X_n;YR_WR_Z|B_WB_ZB')+I_h(X_n;R'|WZYB') \nonumber \\
&\geq I_h(X_n;R_WR_Z|B_WB_ZB')+I_h(X_n;Y|WZB') \nonumber \\
&\geq I_h(X_n;Y|ZWB') \label{eq:sigmaninProof}
\end{align}
}}

Now, if $B'=\emptyset$ then we are done because $h(\sigma_n)\geq h(X_n;Y|ZW)=h(\tau_2)$ and by the induction hypothesis if $\Delta_n\models \Sigma' \implies \tau_1$ then $h(\Sigma')\geq h(\tau_1)$.
So assume that $B'\neq \emptyset$, and consider the following set of atoms:
{\small{
\[
A \eqdef \iset(X_n) \cap \left(\bigcup_{y\in Y}\iset(y)\right) \cap \left(\bigcup_{b\in B'}\iset(b)\right) \cap \bigcap_{X\in ZW}\isetc(X)
\]}}
We note that $\iset(\tau_2)\supseteq A$. By our assumption that $\sigma_n=(X_n;R|B_WB_ZB')$, then $A\cap \iset(\sigma_n)=\emptyset$. Since $\Delta_n\models \Sigma \implies \tau_2$ then by Lemma~\ref{lem:excludeSigma}, it must hold that $\iset(\Sigma')\supseteq A$. Furthermore, since $X_n\notin \vars(\Sigma')$ then it must hold that $\iset(\Sigma')\supseteq A'$ where:
{\small{
\[
A' \eqdef \left(\bigcup_{y\in Y}\iset(y)\right) \cap \left(\bigcup_{b\in B'}\iset(b)\right) \cap \bigcap_{X\in ZW}\isetc(X)
\]}}
Denote by $\tau_3\eqdef (Y;B'|ZW)$, and hence $\iset(\tau_3)=A'$ (see Table~\ref{tab:ImeasureSummary}).
In particular, $\iset(\Sigma')\supseteq \iset(\tau_3)$, and by Lemma~\ref{lem:implicationInclusion} we have that $\Delta_n\models \Sigma' \implies \tau_3$. Since neither $\Sigma'$ nor $\tau_3$ mention $X_n$, then by the induction hypothesis~\eqref{eq:inductionImeasure}, we have that $\imeasure(\iset(\Sigma'))\geq \imeasure(\iset(\tau_3))$.

Since $\Delta_n\models \Sigma \implies \tau_1$, and since $X_n\notin \vars(\tau_1)$, then by the argument of case 1 we have that $\Delta_n\models \Sigma' \implies \tau_1$, and hence by Lemma~\ref{lem:implicationInclusion} that $\iset(\Sigma')\supseteq \iset(\tau_1)$.
Now, by the previous reasoning, we have also have that $\iset(\Sigma')\supseteq \iset(\tau_3)$. 
By noting that $\iset(\tau_1)\cap \iset(\tau_3)=\emptyset$, and applying Lemma~\ref{lem:excludeSigma}, we get that $ \iset(\Sigma'){\setminus}\iset(\tau_3)\supseteq \iset(\tau_1)$.
Applying the induction hypothesis~\eqref{eq:inductionImeasure}, we get that $\imeasure\left(\iset(\Sigma'){\setminus}\iset(\tau_3)\right)\geq \imeasure(\iset(\tau_1))$. Now, since $\imeasure$ is set-additive, and $\iset(\Sigma')\supseteq \iset(\tau_3)$, we get that $\imeasure(\iset(\Sigma'))-\imeasure(\iset(\tau_3))\geq \imeasure(\iset(\tau_1))$. And, by the one-to-one correspondence between Shannon's information measures and the I-measure (Theorem~\ref{thm:YeungUniqueness}), we get that $h(\Sigma')-h(\tau_3)\geq h(\tau_1)$.

Now, from~\eqref{eq:sigmaninProof} we have that $h(\sigma_n)\geq I_h(X_n;Y|ZWB')$. By applying the chain rule:
{
	\small{
\[
\underbrace{I_h(X_n;Y|ZWB')}_{\leq h(\sigma_n)}+\underbrace{I_h(Y;B'|ZW)}_{=h(\tau_3)}=I_h(B'X_n;Y|WZ)\geq h(\tau_2)
\]
}}
Overall, we get that:
{
	\small{
\[
I_h\underbrace{(W;Y|Z)}_{\tau_1}+I_h\underbrace{(X_n;Y|WZ)}_{\tau_2}\leq h(\Sigma')-h(\tau_3)+h(\tau_3)+h(\sigma_n)=h(\Sigma)
\]}}
as required.

\subsection*{Tightness of Bound}
Consider the probability distribution $P$ over $\Omega=\set{X_1,\dots,X_n}$, and suppose that the following recursive set of CIs holds in $P$:
\begin{equation}
	\Sigma = \set{(X_1;X_i|X_2\dots X_{i-1}): i\in \set{2,\dots,n}}
\end{equation}
Let $\tau=(X_1;X_2X_3\dots X_n)$. It is not hard to see that by the chain rule:
\begin{equation}
	\label{eq:tightBound}
I(X_1;X_2X_3\dots X_n)=\sum_{i=2}^nI(X_1;X_i|X_2\dots X_{i-1})=h(\Sigma)
\end{equation}
Hence, $\Sigma \implies_{EI} \tau$, and the bound of~\eqref{eq:tightBound} is tight.
\section{Approximate Implication for Marginal CIs}
\label{sec:marginalProof}
In this section, we prove Theorem~\ref{thm:marginal}.
Let $\Sigma$ be a set of marginal mutual information terms, and let $\tau=(A;B|D)$ such that $\Gamma_n{\models_{EI}}\Sigma{\implies}\tau$. Then, by the chain rule~\eqref{eq:ChainRuleMI}, $\tau$ can be written as a sum of at most $|A||B|$ elemental CIs $(a;b|C)$. In Lemma~\ref{lem:marginalExistsSigma} we show that for every such elemental triple $(a;b|C)$, there exists a marginal $(X;Y){\in}\Sigma$ such that $XY{\supseteq}abC$, $a{\in} X$, and $b{\in}Y$. Consequently, from Lemma~\ref{lem:chainRuleTechnicalLemma}, we get that $h(\Sigma){\geq}I(X;Y){\geq}I(a;b|C)$. Hence, it follows from lemma~\ref{lem:marginalExistsSigma} that $|A||B|h(\Sigma){\geq}h(\tau)$, and this will complete the proof for Theorem~\ref{thm:marginal}.

\begin{lemma}
	\label{lem:marginalExistsSigma}
	Let $\Sigma$ be a set of marginal mutual information terms, and let $\tau=(a;b|C)$ be an elemental mutual information term. The following holds:
	
	\begin{tabular}{ccc}
		\multirow{2}{*}{$\Gamma_n \models_{EI} \Sigma \implies \tau$} & \multirow{2}{*}{iff} & $\exists (X;Y)\in \Sigma:$ \\
		&& $ XY \supseteq abC  \text{ and } a\in X, b\in Y$
	\end{tabular}
\eat{
	\begin{align}
		\label{eq:marginalExistsSigma}
		\Gamma_n \models_{EI} \Sigma \implies \tau && \text{iff} && \substack{\exists (X;Y)\in \Sigma: \\ XY \supseteq abC  \text{ and } a\in X, b\in Y}
	\end{align}
}
\end{lemma}
\begin{proof}
	We prove by induction on $|C|$. When $|C|=0$ then $\tau=(a;b)$. Consider the atom:
	\begin{equation}
		\label{eq:marginalAtomt}
		t \eqdef \iset(a)\cap \iset(b)\bigcap_{y{\in}\Omega{\setminus}ab}\isetc(y)
	\end{equation}
	Clearly, $t{\in}\iset(\tau)$. Suppose, by way of contradiction, that for every $\sigma=(X;Y){\in}\Sigma$ it holds that $ab\cap X =\emptyset$ or $ab\cap Y=\emptyset$. If, without loss of generality, we assume the former then clearly $t \notin \iset(\sigma)$ because all of the RVs in $X$ appear in negative form in the atom $t$. If this is the case for all $\sigma \in \Sigma$, then $t \notin \iset(\Sigma)$, and $\iset(\tau)\not\subseteq \iset(\Sigma)$. But then,  by Corollary~\ref{corr:inclusionGamman}, it cannot be that $\Gamma_n\models_{EI} \Sigma \implies \tau$, a contradiction.
	
	So, we assume correctness for elemental terms $(a;b|C)$ where $|C|{\leq} k{-}1$, and prove for $|C|{=}k$.
	Since $\Gamma_n\models_{EI} \Sigma \implies \tau$, then by Lemma~\ref{lem:implicationCovers} there exists a mutual information term $\sigma{=}(X;Y){\in} \Sigma$ such that $XY {\supseteq} abC$. Hence, we denote  $C{=}C_XC_Y$, where $C_X{=}X{\cap}C$ and $C_Y{=}Y{\cap}C$. There are two cases.
	If $\sigma=(aC_XX_0;bC_YY_0)$ then, by Lemma~\ref{lem:chainRuleTechnicalLemma}, we have that $h(\sigma){\geq}h(\tau)$, and we are done.\eat{
	\begin{align*}
		h(\sigma)&=I(aC_XX_0;bC_YY_0)\\
		&\geq I(aC_X;bC_Y) \geq  I(a;b|C_XC_Y)=h(\tau)		
	\end{align*}
	}

	Otherwise, w.l.o.g, $\sigma=(abC_XX_0;C_YY_0)$.	
	By item 2 of Lemma~\ref{lem:implicationCovers}, it holds that $C_Y{\neq}\emptyset$. 
	\eat{
	If this is not the case, then for the parity function 
	$a=b+\sum_{c\in C}c \mod 2$	
	we have that $h(a;b|C)=h(aC)+h(bC)-h(C)-h(abC)=2+|C|-(|C|+1)=1\neq 0$, while $h(\sigma)=h(abCX_0)+h(Y_0)-h(abCX_0Y_0)=h(abC)+h(X_0)+h(Y_0)-h(abC)-h(X_0)-h(Y_0)=0$\footnote{The reasoning is detailed in the proof of Lemma~\ref{lem:implicationCovers}}. Additionally, $h(\sigma')=0$ for all triples $\sigma'\in \Sigma$ where $\vars(\sigma')\not\supseteq \vars(\tau)$	(see proof of Lemma~\ref{lem:implicationCovers}). Hence, we arrive at a contradiction that $\Gamma_n{\models_{EI}}\Sigma{\implies}\tau$. Therefore, we can assume that there exists a triple $\sigma=(abC_XX_0;C_YY_0)\in \Sigma$ such that $C_Y\neq \emptyset$, and $C_X\subset C$.
	}

	We define:
	\begin{align}
		\alpha_1\eqdef(a;C_Y|C_X) && \alpha_2\eqdef(a;C_Y|bC_X)
	\end{align}	
	By Lemma~\ref{lem:chainRuleTechnicalLemma}, we have that $h(\sigma)\geq h(\alpha_1)$ and $h(\sigma)\geq h(\alpha_2)$, and thus $\Gamma_n\models_{EI}\Sigma\implies \set{\alpha_1,\alpha_2}$.
	Noting that $\tau = (a;b|C_XC_Y)$, we have that $\Gamma_n\models_{EI}\Sigma\implies (a;b|C_XC_Y)$. By the chain rule (see~\eqref{eq:ChainRuleMI}) we have that $\Sigma$ implies:
	$$
		(a;C_Y|C_X), (a;b|C_XC_Y) \implies (a;bC_Y|C_X) \implies (a;b|C_X)
	$$
	In other words, we have that $\Gamma_n \models_{EI} \Sigma \implies (a;b|C_X)$.
	
	By item 2 of Lemma~\ref{lem:implicationCovers} it holds that $C_Y\neq \emptyset$. Hence, $C_X{\subsetneq} C$. Therefore, by the induction hypothesis, there exists an $\alpha_3\eqdef(aC_X^1Z_1; bC_X^2Z_2)\in \Sigma$ where $C_X{=}C_X^1C_X^2$. 
	In particular, by Lemma~\ref{lem:chainRuleTechnicalLemma}, we have that $\alpha_3 \implies \alpha_4\eqdef(a;b|C_X)$, and $h(\alpha_4) \leq h(\alpha_3)$ where $\alpha_3 \in \Sigma$.
	Furthermore, by our assumption (i.e., that $\sigma{=}(abC_XX_0;C_YY_0)$), then $\sigma$ and $\alpha_3$ are distinct.
	Consequently, we get that:
	\begin{align}
		&\underbrace{I(a;b|C_X)}_{\leq h(\alpha_3)}+\underbrace{I(a;C_Y|bC_X)}_{\leq h(\sigma)}\nonumber \\ 
		&{=}I(a;bC_Y|C_X) \geq I(a;b|C_XC_Y)=h(\tau)
	\end{align}
	Overall, we get that $h(\tau)\leq h(\alpha_3)+h(\sigma)\leq h(\Sigma)$ because $\alpha_3,\sigma\in \Sigma$  are distinct, by our assumption.	
	This completes the proof.	
\end{proof}

\section{Conclusion and Discussion}
\label{sec:conclusion}
We study the approximation variant of the well known implication problem, and showed that $d$-separation, the popular inference system used to derive CIs in Bayesian networks, continues to be sound and complete for inferring approximate CIs. We prove a tight approximation factor of $1$ for the case of recursive CIs, and an approximation factor that depends on the size of the implicate for marginal CIs.

The question that remains is whether there are other classes of CIs that admit a $\lambda$-relaxation for a bounded $\lambda$. Previous work has shown that without making any assumptions on the antecedents or the inference system, the answer is negative~\cite{DBLP:conf/icdt/KenigS20}, and when the inference system is the polymatroid inequalities (or equivalently, the semigraphoid axioms) then the bound is $(2^n)!$. 
Despite these negative results, when the set of antecedents fall into certain classes, then they \e{do} admit bounded relaxation. This is the case for saturated CIs~\cite{DBLP:conf/icdt/KenigS20}, which are the foundation for undirected PGMs. It has been shown that the semigraphoid axioms are sound and complete for deriving constraints from saturated CIs~\cite{GeigerPearl1993}. The semigraphoid axioms are also sound and complete for sets of CIs whose cardinality is at most two~\cite{DBLP:journals/amai/Studeny97}, and for the \e{enhanced recursive set} which is a combination of CIs corresponding to a DAG along with functional dependencies~\cite{DBLP:journals/networks/GeigerVP90}. We conjecture that these two sets of CIs also admit a bounded relaxation.

As part of future work we intend to empirically evaluate the extent to which our approach can be applied to the task of extracting the structure of PGMs from observational data. We intend to evaluate our approach along two measures. First, how close the learned model matches the empirical distribution induced by the observed data, and second, how it compares in terms of both accuracy and efficiency to constraint-based algorithms that perform statistical independence tests~\cite{DBLP:journals/jmlr/ColomboM14,DBLP:conf/pgm/ScutariGG18}.

\bibliographystyle{abbrvnat}
\bibliography{main}

\end{document}